\documentclass[english,onecolumn,final,11pt]{IEEEtran} 

\usepackage[T1]{fontenc}
\usepackage[latin9]{inputenc}
 \usepackage{amsthm}
\usepackage{amsmath}
\usepackage{amssymb}
\usepackage{graphicx}
\usepackage{dsfont}
 \usepackage{cite}
\usepackage{amsmath}
\usepackage{array}
\usepackage{multirow}
\usepackage{caption}
\usepackage{color}

\usepackage{multicol}

\theoremstyle{plain}

\theoremstyle{plain}

\theoremstyle{plain}
\newtheorem{lemma}{\protect\lemmaname}
\theoremstyle{plain}
\newtheorem{theorem}{\protect\theoremname}
\theoremstyle{plain}
  
\theoremstyle{definition}

\theoremstyle{definition}

\theoremstyle{definition}

\makeatother
  
\usepackage{babel} 

\providecommand{\claimname}{Claim}
\providecommand{\lemmaname}{Lemma}
\providecommand{\propositionname}{Proposition}
\providecommand{\theoremname}{Theorem}
\providecommand{\corollaryname}{Corollary} 
\providecommand{\definitionname}{Definition}
\providecommand{\assumptionname}{Assumption}
\providecommand{\remarkname}{Remark}

\newcommand{\openone}{\mathds{1}}

\newcommand{\vbar}{\overline{v}}

\newcommand{\Hc}{\mathcal{H}}

\newcommand{\xv}{\mathbf{x}}

\newcommand{\yv}{\mathbf{y}}

\newcommand{\Fc}{\mathcal{F}}

\newcommand{\EE}{\mathbb{E}}
\newcommand{\PP}{\mathbb{P}}
\newcommand{\RR}{\mathbb{R}}
\newcommand{\ZZ}{\mathbb{Z}}

\newcommand{\Rc}{\mathcal{R}}

\newcommand{\iv}{\mathbf{i}}

\newcommand{\bzero}{\mathbf{0}}
\newcommand{\bone}{\mathbf{1}}

\usepackage{amsmath}
\usepackage{amssymb}
\usepackage{natbib}
\usepackage{graphicx}
\usepackage{url}
\usepackage{algorithm2e}
\usepackage{float}
\usepackage{color}
\usepackage{array}
\usepackage{hyperref}
\usepackage{setspace}
\usepackage[margin=1.1in]{geometry}

\setcitestyle{numbers,square}

\newcommand{\Dbar}{\overline{D}}
\newcommand{\kSE}{k_{\text{SE}}}
\newcommand{\kMat}{k_{\text{Mat\'ern}}}

\title{Lower Bounds on Regret for Noisy\\ Gaussian Process Bandit Optimization}
\usepackage{times}

\author{Jonathan Scarlett, Ilija Bogunovic, and Volkan Cevher \\ [3mm]
    \small Laboratory for Information and Inference Systems (LIONS), EPFL \\
    \{jonathan.scarlett, ilija.bogunovic, volkan.cevher\}@epfl.ch}

\begin{document}

\maketitle

\begin{abstract}
    In this paper, we consider the problem of sequentially optimizing a black-box function $f$ based on noisy samples and bandit feedback.  We assume that $f$ is smooth in the sense of having a bounded norm in some reproducing kernel Hilbert space (RKHS), yielding a commonly-considered non-Bayesian form of Gaussian process bandit optimization.  We provide algorithm-independent lower bounds on the simple regret, measuring the suboptimality of a single point reported after $T$ rounds, and on the cumulative regret, measuring the sum of regrets over the $T$ chosen points.  
    
    For the isotropic squared-exponential kernel in $d$ dimensions, we find that an average simple regret of $\epsilon$ requires $T = \Omega\big(\frac{1}{\epsilon^2} (\log\frac{1}{\epsilon})^{d/2}\big)$, and the average cumulative regret is at least $\Omega\big( \sqrt{T(\log T)^{d/2}} \big)$, thus matching existing upper bounds up to the replacement of $d/2$ by $2d+O(1)$ in both cases.  For the Mat\'ern-$\nu$ kernel, we give analogous bounds of the form $\Omega\big( (\frac{1}{\epsilon})^{2+d/\nu}\big)$ and $\Omega\big( T^{\frac{\nu + d}{2\nu + d}} \big)$, and discuss the resulting gaps to the existing upper bounds.
\end{abstract}



%
%

\section{Introduction}

The problem of sequentially optimizing a black-box function based on noisy bandit feedback has recently attracted a great deal of attention, and finds applications in robotics, environmental monitoring, and hyperparameter optimization in machine learning, just to name a few.  In order to make this problem tractable, one must place smoothness assumptions on the function.  Gaussian processes (GPs) provide a versatile means for doing this, capturing the smoothness properties through a suitably-chosen kernel.

Within the GP framework, there are two distinct viewpoints: Bayesian and non-Bayesian.  In the Bayesian viewpoint, one assumes that the underlying function is random according to a GP distribution with the specified kernel.  On the other hand, in the non-Bayesian viewpoint, one treats the function as fixed and unknown, and only assumes that it is has a low norm in the reproducing kernel Hilbert space (RKHS) corresponding to the kernel.

Upper bounds on the \emph{regret}, a widely-adopted performance measure (see Section \ref{sec:setup}), have been given for a variety of practical GP bandit optimization\footnote{We use the terminology \emph{Gaussian process bandit optimization} to refer collectively to the Bayesian and non-Bayesian settings, whereas the former is typically referred to as \emph{Bayesian optimization}.} algorithms, including upper confidence bound (GP-UCB) \cite{Sri09} and various alternatives \cite{Rus14,Wan16,Bog16a}.  However, to our knowledge, there are no existing \emph{algorithm-independent lower bounds} on the regret in the noisy setting, thus making it unclear to what extent the upper bounds can be improved.

In this paper, we provide such lower bounds in the \emph{non-Bayesian} setting, focusing on two specific widely-used kernels: squared exponential (SE) and M\'atern.  For the SE kernel, our lower bound nearly matches existing upper bounds, whereas the gap is seen to be more significant in the case of the M\'atern kernel.  As a result, the latter deserves further study both in terms of the lower bounds and the existing algorithms; see Section \ref{sec:results} for further discussion.

\subsection{Setup} \label{sec:setup}

We consider the optimization of a function $f$ defined on a compact domain $D = [0,1]^d$.  The smoothness is dictated by a \emph{kernel function} $k(x,x')$ \cite{Ras06}; specifically, we assume that the corresponding RKHS norm is bounded as $\|f\|_{k} \le B$.  The set of all such functions is denoted by $\Fc_k(B)$, and for a given $f \in \Fc_k(B)$, we let $x^*$ denote an arbitrary maximizer of $f$.  In the $t$-th round, a point $x_t$ is selected, and a noisy sample $y_t = f(x_t) + z_t$ is observed, with $z_t \sim N(0,\sigma^2)$ for some $\sigma^2 > 0$.  We assume independence among the noise terms at differing time instants.

We consider two widely considered performance metrics:
\begin{itemize}
    \item {\bf Simple regret:} At the end of $T$ rounds, an additional point $x^{(T)}$ is reported, and the simple regret incurred is $r^{(T)} = f(x^*) - f(x^{(T)})$.  We seek to characterize how large $T$ must be to permit $\EE[r^{(T)}] \le \epsilon$.
    \item {\bf Cumulative regret:} At the end of $T$ rounds, the cumulative regret incurred is $R_T = \sum_{t=1}^T r_t$, where $r_t = f(x^*) - f(x_t)$.  We seek to provide lower bounds on $\EE[R_T]$.
\end{itemize}

We focus on two commonly-considered kernels, namely, squared exponential (SE) and Mat\'ern \cite{Ras06}: 
\begin{align}
    \kSE(x,x') &= \exp \bigg(- \dfrac{\|x - x'\|^2}{2l^2} \bigg) \label{eq:kSE} \\ 
    \kMat(x,x') &= \dfrac{2^{1-\nu}}{\Gamma(\nu)} \bigg(\dfrac{\sqrt{2\nu}\|x - x'\|}{l}\bigg)^{\nu}  J_{\nu}\bigg(\dfrac{\sqrt{2 \nu}\|x - x'\|}{l} \bigg), \label{eq:kMat}
    \end{align}
where $l>0$ denotes the length-scale, $\nu > 0$ is an additional parameter that dictates the smoothness,  and $J_{\nu}$ denotes the modified Bessel function.  Although we only consider these specific kernels, the approach that we take can also potentially be applied to other stationary kernels.

\subsection{Related Work} \label{sec:previous_work}

A recent review of Gaussian process bandit optimization is given in \cite{Sha16}. Upper bounds on the regret for specific algorithms are given for the Bayesian and non-Bayesian settings with cumulative regret in \cite{Sri09,Wan14a,Rus14}, and with simple regret in \cite{Con13,Bog16a}.  Of these, \cite{Wan14a} is the only one that did not assume perfect knowledge of the kernel, but this distinction will not be important for our purposes, since our lower bounds hold in either case.  Bounded RKHS norm assumptions have also been used in other problems, such as uniform function approximation \cite{Ras12,Ras14}.  

We are not aware of any existing lower bounds for the noisy setting.  However, \cite{Bul11} is related to our work, providing tight lower bounds on the simple regret for the M\'atern kernel in the {\em noiseless} setting.  While our high-level approach is similar to \cite{Bul11} in the sense of considering a class of {\em needle-in-haystack} type functions that are difficult to distinguish, there are a number of additional challenges that we need to address:
\begin{enumerate}
    \item Incorporating noise is non-trivial, and it is unclear {\em a priori} the extent to which it affects the regret bounds.  For instance, in the Bayesian setting, $O(1)$ cumulative regret is possible for noiseless observations \cite{Fre12,Kaw15}, whereas the best known bounds for the noisy case are between $O(\sqrt{T})$ and $o(T)$ \cite{Sri09}.
    \item We consider not only the simple regret, but also the cumulative regret, which is not captured by the framework of \cite{Bul11}.  
    \item While \cite[Thm.~1]{Bul11} gives a tight lower bound for the M\'atern kernel, no analogous result is given for the widely-adopted SE kernel.  Moreover, an inspection of the analysis in \cite{Bul11} reveals that the {\em compactly supported} functions used therein have an infinite RKHS norm for the SE kernel, and thus cannot be used.  To circumvent this issue, we consider a different class of functions that are compactly supported {\em in frequency domain}, and hence have unbounded support in the spatial domain $\RR^d$.  This introduces additional challenges into the analysis, since we can no longer directly state that most samples are uninformative as a result of the sampled function value being exactly zero.
\end{enumerate}
Another related work on the noiseless setting is  \cite{Gru10}, which provided upper and lower bounds for the Bayesian case under a H\"older-continuity assumption on the kernel.  

While multi-armed bandit problems \cite{Bub12,Aue95,Aud10} typically consist of \emph{finite} action spaces, some continuous variants have been proposed \cite{Kle04,Kle08}.  However, these make significantly different smoothness assumptions of the \emph{Lipschitz} variety, and as a result, the results and analysis techniques do not apply to our setting of bounded RKHS norm.  In particular, Lipschitz properties do not imply bounds on the RKHS norm, nor vice versa.

\section{Main Results} \label{sec:results}

In all of our results, we assume that the dimension $d$ in the definition of $D = [0,1]^d$, as well as the parameters $l$ and $\nu$ in the kernels in \eqref{eq:kSE}--\eqref{eq:kMat}, behave as $\Theta(1)$.  In contrast, we allow the RKHS norm upper bound $B$ and the noise level $\sigma$ to scale generally.  We let $O^*(\cdot)$ denote asymptotic expressions up to dimension-independent logarithmic factors, e.g., $4\sqrt{T}(\log T)^{d+3} = O^*(\sqrt{T}(\log T)^d)$.

Our main results are expressed in terms of the \emph{average} regret.  However, it is straightforward to modify the proofs in order to show that the scaling laws are the same when considering regret bounds that hold with constant (but arbitrarily close to one) probability.  See Section \ref{sec:HIGH_PROB} for details.

We first provide our main result for the simple regret.

\begin{theorem} \label{thm:instantaneous}
    \emph{(Simple Regret)}
    Fix $\epsilon \in \big(0,\frac{1}{2}\big)$, $B > 0$, and $T \in \ZZ$.   Suppose there exists an algorithm that, for any $f \in \Fc_k(B)$, achieves an average simple regret $\EE[r^{(T)}] \le \epsilon$ after $T$ rounds.  Then, provided that $\frac{\epsilon}{B}$ is sufficiently small, we have the following:
    \begin{enumerate}
        \item For \emph{$k = \kSE$}, it is necessary that
        \begin{equation}
            T = \Omega\bigg( \frac{\sigma^2}{\epsilon^2} \Big(\log\frac{B}{\epsilon}\Big)^{d/2} \bigg). \label{eq:inst_se}
       \end{equation}
        \item For \emph{$k = \kMat$}, it is necessary that
        \begin{equation}
            T = \Omega\bigg( \frac{\sigma^2}{\epsilon^2} \Big(\frac{B}{\epsilon}\Big)^{d/\nu} \bigg). \label{eq:inst_Mat}
       \end{equation}
    \end{enumerate}
\end{theorem}

Our main result for the cumulative regret is as follows.

\begin{theorem} \label{thm:cumulative}
    \emph{(Cumulative Regret)}
    Fix $B > 0$ and $T \in \ZZ$.   Suppose there exists an algorithm that, for any $f \in \Fc_k(B)$, achieves an average cumulative regret of $\EE[R_T]$ after $T$ rounds.  Then, we have the following:
    \begin{enumerate}
        \item For $k = \kSE$, it is necessary that
        \begin{equation}
            \EE[R_T] = \Omega\Bigg( \sqrt{T\sigma^2 \Big(\log \frac{B^2 T}{\sigma^2} \Big)^{d/2} } \Bigg) \label{eq:cumul_SE}
       \end{equation}
       provided that $\frac{\sigma}{B} = O\big(\sqrt{T}\big)$ with a sufficiently small implied constant.
        \item For $k = \kMat$, it is necessary that
        \begin{equation}
            \EE[R_T] = \Omega\bigg( B^{\frac{d}{2\nu + d}} \sigma^{\frac{2\nu}{2\nu + d}} T^{\frac{\nu + d}{2\nu + d}} \bigg) \label{eq:cumul_Mat}
       \end{equation}
       provided that $\frac{\sigma}{B} = O\big(\sqrt{T}\big)$ with a sufficiently small implied constant.
    \end{enumerate}
\end{theorem}

We note that the condition $\frac{\sigma}{B} = O\big(\sqrt{T}\big)$ is assumed for technical reasons, and is quite mild.  In particular, below we will focus primarily on the case that $\sigma$ and $B$ are constants that do not scale with $T$, in which case this condition is trivially satisfied.

\textbf{Comparisons to Upper Bounds:}
The best known upper bounds on the cumulative regret were given for the upper confidence bound (GP-UCB) algorithm \cite{Sri09}, where bounded noise was assumed.  While Gaussian noise has unbounded support, it is bounded by $O(\log T)$ for all $t=1,\dotsc,T$ with high probability, and hence the bounded noise results transfer to the Gaussian setting at the expense of additional dimension-independent logarithmic factors.  

A summary of the comparisons is shown in Table \ref{tbl:summary}, and we proceed by discussing the entries in more detail.\footnote{This version of Table \ref{tbl:summary} corrects some minor mistakes in the previous version of this manuscript, where a few of the exponents $(\log(\cdot))^{d/2}$ vs.~$(\log(\cdot))^{d}$ vs.~$(\log(\cdot))^{2d}$ were off by a factor of two for the SE kernel.}  We first compare our Theorem \ref{thm:cumulative} (cumulative regret) to \cite{Sri09} in the case that $\sigma^2$ and $B$ behave as $\Theta(1)$; we will later discuss the dependence on these parameters.  For fixed $\sigma^2$, the upper bound in \cite{Sri09} is of the form
\begin{equation}
    R_T = O^*\big(\sqrt{TB\gamma_T + T\gamma_T^2}\big), \label{eq:cumul_upper}
\end{equation} 
where, letting $I(X;Y)$ denote the mutual information \cite{Cov01}, we define
\begin{equation} 
    \gamma_T = \max_{x_1,\dotsc,x_T} \max_{S \,:\, |S| = T} I(f;\mathbf{y}_S), \qquad f \sim \mathcal{GP}(0,k). \label{eq:gamma_def}
\end{equation}
This represents the maximum amount of information that a set of $T$ noisy observations $\yv_S = (y_1,\dotsc,y_T)$ (corresponding to $\xv_S = (x_1,\dotsc,x_T)$) can reveal about a zero-mean Gaussian process $f$ with kernel $k$. We have $\gamma_T = O\big( (\log T)^{d+1}\big)$ for the squared-exponential kernel, and $\gamma_T = O^*\big( T^{\frac{d(d+1)}{2\nu + d(d+1)}}\big)$ for the Mat\'ern kernel \cite{Sri09}.

Hence, for $\kSE$, the upper bound in \eqref{eq:cumul_upper} is $O^*\big( \sqrt{T (\log T)^{2d}} \big)$ and the lower bound in \eqref{eq:inst_se} is $\Omega\big( \sqrt{T (\log T)^{d/2}} \big)$, so the two coincide up to the factor of $2$ vs.~$\frac{1}{2}$ in the exponent, as well as the (few) extra $\log T$ factors hidden in the $O^*(\cdot)$ notation.  

While \cite{Sri09} focused on the cumulative regret, one can obtain simple regret bounds by noting that it it always possible to have a simple regret no higher than the normalized cumulative regret \cite{Bub09}; this yields the condition $\frac{T}{B\gamma_T + \gamma_T^2} \ge \frac{C}{\epsilon^2}$ for $\epsilon$-optimality, where $C = O^*(1)$ (see also \cite{Con13,Bog16a}).  By substituting $\gamma_T = O\big( (\log T)^{d+1}\big)$ and rearranging, we find that for the SE kernel this condition is of the form $T \ge \frac{C'}{\epsilon^2} \big(\log\frac{1}{\epsilon}\big)^{2d}$ for some $C' = O^*(1)$, again matching \eqref{eq:inst_se} up to the factor of  $2$ vs.~$\frac{1}{2}$ in the exponent.

For the Mat\'ern kernel, the gaps are more significant.  In particular, substituting $\gamma_T = O^*\big( T^{\frac{d(d+1)}{2\nu + d(d+1)}}\big)$, we find that the cumulative regret upper bound in \eqref{eq:cumul_upper} is only sublinear when $\frac{d(d+1)}{2\nu + d(d+1)} < \frac{1}{2}$, or equivalently $2\nu-d(d+1)>0$.  A similar statement holds for having a non-void condition on the simple regret.  This suggests that these existing upper bounds may be loose, since, at least in the authors' judgment,  sublinear regret should be possible for all $\nu$ and $d$.  This poses an interesting open problem for future work.

In contrast, if the right-hand side of \eqref{eq:cumul_upper} could be improved to $O^*\big( \sqrt{TB\gamma_T} \big)$, then the upper bound would be sublinear for all $d$ and $\nu$.  We conjecture that such a bound holds (as is the case in the Bayesian setting \cite{Sri09}), but we do not currently have a proof.  We proceed by discussing the bounds that would arise if this conjecture were true, as summarized in the middle column of Table \ref{tbl:summary}.

\begin{table}
    \hspace*{-0.6cm}
    \begin{centering}
    \begin{tabular}{|>{\centering}m{3.4cm}|>{\centering}m{3.5cm}|>{\centering}m{3.6cm}|>{\centering}m{3cm}|}
    \hline 
     & {\small Upper bound \\ (\cite{Sri09})} & {\small Conjectured upper bound \\ (see Section \ref{sec:results})} & {\small Lower bound \\ (Theorems \ref{thm:instantaneous}--\ref{thm:cumulative})} \tabularnewline
    \hline 
    \hline 
    {\small SE kernel} \\ {\small Cumulative regret} &
         $O^{*}\Big(\sqrt{T(\log T)^{2d}}\Big)$ 
        & $O^{*}\Big(\sqrt{T(\log T)^{d}}\Big)$ 
        & $\Omega\Big(\sqrt{T(\log T)^{d/2}}\Big)$
    \tabularnewline
    \hline 
    {\small Mat\'ern kernel}  \\ {\small Cumulative regret}
        & $O^{*}\bigg(T^{\frac{1}{2}\cdot\frac{2\nu+3d(d+1)}{2\nu+d(d+1)}}\bigg)$
        & $O^{*}\bigg(T^{\frac{\nu+d(d+1)}{2\nu+d(d+1)}}\bigg)$ 
        & $\Omega\bigg(T^{\frac{\nu+d}{2\nu+d}}\bigg)$ 
    \tabularnewline
    \hline
    {\small SE kernel} \\ {\small Time to simple regret $\epsilon$}
        & $O^{*}\bigg(\frac{1}{\epsilon^{2}}\Big(\log\frac{1}{\epsilon}\Big)^{2d}\bigg)$ 
        & $O^{*}\bigg(\frac{1}{\epsilon^{2}}\Big(\log\frac{1}{\epsilon}\Big)^{d}\bigg)$  
        & $\Omega\bigg(\frac{1}{\epsilon^{2}}\Big(\log\frac{1}{\epsilon}\Big)^{d/2}\bigg)$
    \tabularnewline
    \hline 
    {\small Mat\'ern kernel} \\ {\small Time to simple regret  $\epsilon$}
        & $O^{*}\bigg(\Big(\frac{1}{\epsilon}\Big)^{\frac{2(2\nu+d(d+1))}{2\nu-d(d+1)}}\bigg)$ \\
            (if $2\nu-d(d+1)>0$)  
        & $O^{*}\bigg(\Big(\frac{1}{\epsilon}\Big)^{2+d(d+1)/\nu}\bigg)$
        & $\Omega\bigg(\Big(\frac{1}{\epsilon}\Big)^{2+d/\nu}\bigg)$
    \tabularnewline
    
    \hline 
    
    \end{tabular}
    \par\end{centering}
    
    \protect\protect\caption{Summary of regret bounds for a fixed RKHS norm bound $B>0$ and noise level $\sigma^2>0$. \label{tbl:summary}}
\end{table}

For the squared exponential kernel, the conjecture would partially close the above-mentioned gap on the factor of $2$ in the exponent to $\log T$.  For the Mat\'ern kernel, we would obtain a cumulative regret of $O^*\big(  T^{\frac{\nu + d(d+1)}{2\nu + d(d+1)}} \big)$, hence matching our lower bound up to the replacement of $d$ by $d(d+1)$.  Thus, both bounds would be of the same form, and behave as $T^c$ for some $c$ that tends to one as $d$ grows large.  Analogous observations hold for the simple regret, as summarized in Table \ref{tbl:summary}.

Finally, we discuss the dependence of the bounds on $B$ and $\sigma$, focusing on the squared exponential kernel with cumulative regret for brevity.  Up to the logarithmic term, our bound \eqref{eq:cumul_SE} depends linearly on $\sigma$.  In contrast, the upper bound in \cite{Sri09} contains a multiplicative term $O\big( \sqrt{\frac{1}{\log(1+\sigma^{-2})}} \big)$, which reduces to $O(\sigma)$ when $\sigma$ is large, but tends to zero much more slowly when $\sigma$ is small.  However, it is shown in \cite{Bog16a} that a linear dependence is attainable for the simple regret, and one should expect that such a dependence is similarly attainable for the cumulative regret.  As for the dependence on $B$, the optimal scaling as $B \to \infty$ is unclear.  For fixed $T$, $d$, and $\sigma^2$, our lower bound has dependence $(\log B)^d$, whereas the upper bound of in \eqref{eq:cumul_upper} has dependence $\sqrt{B}$, which is stronger since $d = O(1)$.

The remainder of the paper is devoted to the proofs of Theorems \ref{thm:instantaneous} and \ref{thm:cumulative}.

\section{Construction of a Finite Ensemble of Functions}  \label{sec:ensemble}

We obtain lower bounds on the regret by considering the case that $f$ is uniform on a finite set $\Fc' = \{f_1,\cdots,f_M\}$.  If we can lower bound the regret (of an arbitrary algorithm) averaged over $m \in \{1,\dotsc,M\}$, then it follows that there must exist a particular value of $m$ such that the same lower bound applies to $f_m$.

More specifically, we let $\Fc'$ have the property that any single point $x$ can only be $\epsilon$-optimal for at most one function in this set.  By this property, the optimization problem with simple regret not exceeding $\epsilon$ is equivalent to the correct identification of the index $m \in \{1,\cdots,M\}$.  

Note that for the simple regret, the parameter $\epsilon$ is part of the problem statement, whereas for cumulative regret it is a parameter that we can choose to our liking.  The starting point of the analysis is the same for both regret notions, and throughout, we assume that $\frac{\epsilon}{B}$ is sufficiently small.  This is true by assumption in Theorem \ref{thm:instantaneous}, whereas when we prove Theorem \ref{thm:cumulative}, we will need to check that our choice of $\epsilon$ is consistent with this assumption.

\subsection{Function Class} \label{sec:FUNC_CLASS}

The class that we consider corresponds to a ``needle in haystack'' problem, and is described as follows: Let $g(x)$ be a function on $\RR^d$ with values in $[-2\epsilon,2\epsilon]$, an RKHS norm not exceeding $B$, a peak value of $2\epsilon$ at $x = 0$, and a value strictly less than $\epsilon$ when $\|x\|_{\infty} \ge w$, for some \emph{width} $w > 0$ to be chosen later.  We let each $f_m(x)$ be given by $g(x)$ shifted so that its peak is at a given point, and then cropped to the domain $[0,1]^d$. By forming a $d$-dimensional grid of step size $w$ in each dimension, we can form
\begin{equation}
    M = \bigg\lfloor \bigg( \frac{1}{w} \bigg)^d \bigg\rfloor  \label{eq:M_choice}
\end{equation}
of these while ensuring that any $\epsilon$-optimal point for $f_m$ fails to be $\epsilon$-optimal for any of the other $f_{m'}$.  See Figure \ref{fig:func_class} for an illustration.

It remains to choose $g$, and to determine small we can allow $w$ to be without violating the RKHS norm constraint.  Intuitively, a smaller $w$ gives a ``less smooth'' function and hence a higher RKHS norm, but a smaller amplitude $2\epsilon$ also decreases the RKHS norm in a linear fashion.  Hence, as $\epsilon \to 0$, we can afford to take $w \to 0$.

\begin{figure}
    \begin{centering}
        \includegraphics[width=0.7\columnwidth]{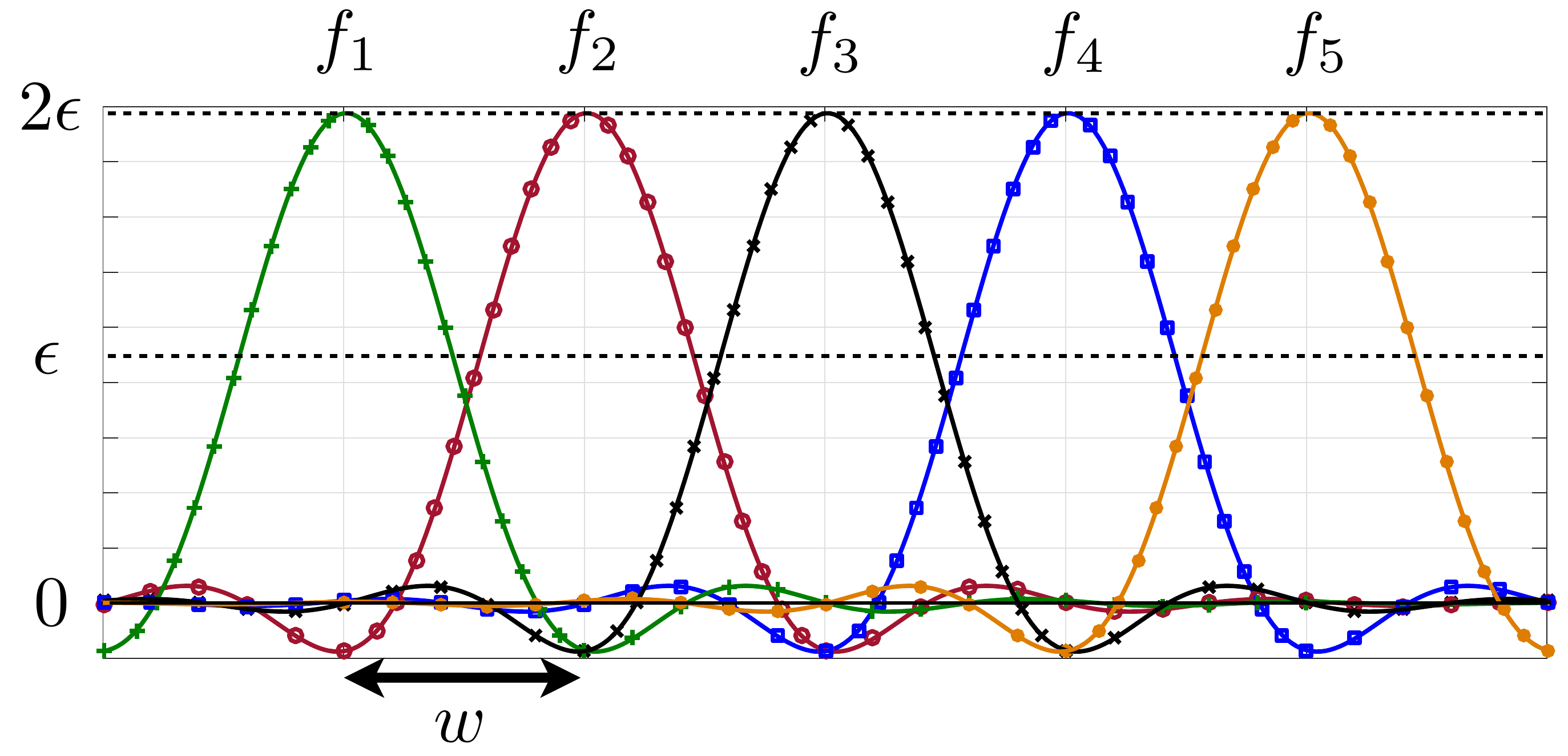}
        \par
    \end{centering}
    
    \caption{Illustration of functions $f_1,\dotsc,f_5$ such that any given point $x \in [0,1]$ can only be $\epsilon$-optimal for one function.  See \eqref{eq:Hf}--\eqref{eq:g_choice} for the specific function being plotted. \label{fig:func_class}}
\end{figure}

Instead of specifying $g$ directly, we specify the Fourier transform $H$ of another function $h$, and then let $g$ be a suitably scaled version of $h$.  Specifically, $H$ is defined to be the multi-dimensional \emph{bump function}:
\begin{equation}
    H(\xi) = 
    \begin{cases}
        \exp\bigg( - \frac{1}{1-\|\xi\|^2} \bigg) & \|\xi\|_2^2 \le 1 \\
        0 & \|\xi\|_2^2 > 1,
    \end{cases} \label{eq:Hf}
\end{equation}
which is  compactly supported and infinitely differentiable.  Let $h(x)$ be the inverse Fourier transform of $H$.  Since $H$ is a fixed function with finite energy, the amplitude of $h(x)$ must vanish as $\|x\|$ grows large; hence, there exists a constant $\zeta$ such that $h(x) < \frac{1}{2} h(0)$ for $\|x\|_{\infty} > \zeta$.  

We choose 
\begin{equation}
    g(x) = \frac{2\epsilon}{h(0)} h\Big( \frac{x\zeta}{w} \Big), \label{eq:g_choice}
\end{equation}
where the parameter $w$ is the same as that in \eqref{eq:M_choice}.  Note that since $H$ is real and symmetric, the maximum of $h$ occurs at zero, and thus the maximum of $g$ is $g(0) = 2\epsilon$, as desired.  We consider $w$ as arbitrary for now, but this will be chosen to ensure that the RKHS norm of $g$ is upper bounded by $B$; see Section \ref{sec:BOUND_RKHS}.  To get an idea of how $g$ behaves, we note that it is precisely the function that was used in producing Figure \ref{fig:func_class} (with $d=1$).

\section{Bounding the RKHS Norm} \label{sec:BOUND_RKHS}

Intuitively, in our function class in Section \ref{sec:FUNC_CLASS}, we would like $M$ to be as large as possible, since this means that there are more functions that need to be distinguished, and a higher lower bound can be obtained.  In this section, we determine how large $M$ can be in our construction while still ensuring that our assumptions are valid, namely, the RKHS norm of each function is upper bounded by $B$, and any given point can only be $\epsilon$-optimal for at most one function. 

\subsection{Properties of RKHS Norms}

The following properties of RKHS norms are well-known; see for example \cite[Sec.~1.5]{Aro50} and \cite[Sec.~3]{Bul11}.  The first property explicitly expresses the RKHS norm on $\RR^d$ in terms of the Fourier transform.

\begin{lemma} \emph{\cite[Sec.~1.5]{Aro50}} \label{lem:norm}
    Consider an RKHS $\Hc$ for functions on $\RR^d$, corresponding to a kernel of the form $k(x,x') = k(\tau)$ with $\tau = x - x'$, and let $K(\xi)$ be the $d$-dimensional Fourier transform of $k(\tau)$.  Then for any $\tilde{f} \in \Hc$ with Fourier transform $\tilde{F}(\xi)$, we have
    \begin{equation}
        \|\tilde{f}\|_{\Hc}^2 = \int \frac{ |\tilde{F}(\xi)|^2 }{ K(\xi) } d\xi.
    \end{equation}
\end{lemma}

The next result relates the RKHS norm on $\RR^d$ to that on a compact subset $D \subseteq \RR^d$, and reveals that the cropping operation described at the start of Section \ref{sec:FUNC_CLASS} cannot increase the RKHS norm.

\begin{lemma} \emph{\cite[Sec.~1.5]{Aro50}} \label{lem:crop}
    Consider two RKHS $\Hc(D)$ and $\Hc(\RR^d)$ for functions on a compact set $D \subseteq \RR^d$ and $\RR^d$ respectively, corresponding to a kernel of the form $k(x,x') = k(\tau)$ with $\tau = x - x'$.  Then for any $\tilde{f} \in \Hc(D)$, we have
    \begin{equation}
        \|\tilde{f}\|_{\Hc(D)} = \inf_{\tilde{g}} \|\tilde{g}\|_{\Hc(\RR^d)},
    \end{equation}
    where the infimum is over all functions $\tilde{g} \in \Hc(\RR^d)$ that agree with $\tilde{f}$ when restricted to $D$.
\end{lemma}

\subsection{Squared Exponential Kernel}

We first consider the isotropic squared exponential kernel with lengthscale $l$, defined in \eqref{eq:kSE}.  Writing $k(x,x') = k(\tau)$ with $\tau = x - x'$, the Fourier transform is given by \cite[Sec.~4.2]{Ras06}
\begin{equation}
    K(\xi) = (2\pi l^2)^{d/2} e^{- 2\pi^2l^2 \|\xi\|_2^2}. \label{eq:fourier_se}
\end{equation}
Using Lemma \ref{lem:norm} and \eqref{eq:fourier_se}, and denoting $a_0 = \frac{2\epsilon}{h(0)}$ and $w_0 = \frac{w}{\zeta}$ (\emph{cf.}, \eqref{eq:g_choice}) for brevity, the RKHS norm of $g$ in $\RR^d$ is given by
\begin{align}
    \|g\|_k^2
        &= \frac{1}{(2\pi l^2)^{d/2}} \int |G(\xi)|^2 e^{ 2\pi^2 l^2 \|\xi\|_2^2 } d\xi \label{eq:g_norm_se} \\
        &=  \frac{a_0^2 w_0^{2d} }{(2\pi l^2)^{d/2}} \int_{ \|w_0 \xi\|_2 \le 1} \exp\bigg( 2\pi^2 l^2  \|\xi\|_2^2 - \frac{2}{ 1-\| w_0 \xi\|^2 } \bigg) d\xi \label{eq:g_norm_se2} \\
        &\le \frac{a_0^2 w_0^{2d}}{(2\pi l^2)^{d/2}} \int_{\|\xi\|_2 \le \frac{1}{w_0}} \exp\big( 2\pi^2 l^2  \|\xi\|_2^2 \big) d\xi \\
        &\le \frac{a_0^2 w_0^{2d}}{(2\pi l^2)^{d/2}} V(w_0^{-1}) \exp\Big( \frac{ 2\pi^2 l^2  }{ w_0^2 } \Big) \label{eq:g_norm_se4} \\
        &\le  \frac{a_0^2}{(2\pi l^2)^{d/2}} \exp\Big( \frac{ 2\pi^2 l^2 }{ w_0^2 } \Big), \label{eq:g_norm_se5}
\end{align}
where:
\begin{itemize}
    \item \eqref{eq:g_norm_se2} follows from \eqref{eq:Hf}--\eqref{eq:g_choice} and the fact that the $d$-dimensional Fourier transform of $a h(x/b)$ is $ab^d H(xb)$;
    \item \eqref{eq:g_norm_se4} follows by upper bounding $\|\xi\|_2^2 \le \frac{1}{w_0^2}$ according to the constraint, and defining $V(w_0^{-1})$ to be the volume of a ball of radius $w_0^{-1}$ in $d$ dimensions:
    \begin{equation}
        V(w_0^{-1}) = \frac{\pi^{d/2}}{\Gamma(\frac{d}{2}-1)} (w_0^{-1})^d, \label{eq:volume}
    \end{equation}
    where $\Gamma(\cdot)$ denotes the gamma function;
    \item \eqref{eq:g_norm_se5} holds when $w_0$ is chosen such that $w_0^{2d} V(w_0^{-1}) \le 1$, as will be done shortly.
\end{itemize}  

Equating \eqref{eq:g_norm_se5} with $B^2$ (so that $\|g\|_k \le B$), and recalling the definition $a_0 = \frac{2\epsilon}{h(0)}$, we find that
\begin{align}
    & B (2\pi l^2)^{d/4} \exp\Big( - \frac{ \pi^2 l^2 }{ w_0^2 } \Big) h(0) = 2\epsilon \\
    &\implies w_0 = \frac{\pi l}{ \sqrt{\log\frac{B (2\pi l^2)^{d/4} h(0)}{2\epsilon} } }. \label{eq:w0_1}
\end{align}
From this choice, we see that $w_0 \to 0$ as $\frac{\epsilon}{B} \to 0$.  Hence, and since $w_0^{2d} V(w_0^{-1}) = O(w_0^d)$, we find that the above assumption $w_0^{2d} V(w_0^{-1}) \le 1$ is indeed true as long as $\frac{\epsilon}{B}$ is sufficiently small.  The latter is assumed in Theorem \ref{thm:instantaneous}, and will be ensured when we choose $\epsilon$ to prove Theorem \ref{thm:cumulative}.

Recalling that $h(x) < \frac{1}{2}h(0)$ for $\|x\|_{\infty} > \zeta$, we have from \eqref{eq:g_choice} that $g(x) < \frac{1}{2}g(0) = \epsilon$ for $\|x\|_{\infty} > \zeta w_0 = w$, as was assumed in the construction of the functions in Section \ref{sec:FUNC_CLASS}.  Hence, from \eqref{eq:w0_1}, we can choose $w = \zeta w_0 = \frac{\zeta \pi l}{ \sqrt{\log\frac{B (2\pi l^2)^{d/4} h(0)}{2\epsilon} } }$ in \eqref{eq:M_choice}, yielding
\begin{equation}
    M = \Bigg\lfloor \Bigg( \frac{ \sqrt{\log\frac{B (2\pi l^2)^{d/4} h(0)}{2\epsilon}} }{\zeta \pi l} \Bigg)^d \Bigg\rfloor. \label{eq:M_se}
\end{equation}
In summary, we have shown that with this choice of $M$, we can construct $M$ functions whose RKHS norm is upper bounded by $B$, whose peak values are $2\epsilon$, and such that any $\epsilon$-optimal point for one function cannot be $\epsilon$-optimal for any of the other $M-1$ functions.

\subsection{Mat\'ern Kernel}

We now consider the Mat\'ern  kernel with parameters $l$ and $\nu$, defined in \eqref{eq:kMat}.  The Fourier transform of $k$ (treated as a function of $\tau = x-x'$) is \cite[Sec.~4.2]{Ras06}
\begin{equation}
    K(\xi) = c_1 \bigg( \frac{2\nu}{l^2} + 4\pi^2 \|\xi\|_2^2 \bigg)^{-(\nu + d/2)},
\end{equation}
where $c_1 = \frac{ 2^d \pi^{d/2} \Gamma(\nu + d/2) (2\nu)^{\nu} }{ \Gamma(\nu) l^{2\nu} }$ is a constant (recall that $d$, $l$, and $\nu$ are assumed to be fixed).  Hence, using Lemma \ref{lem:norm} and  again denoting $a_0 = \frac{2\epsilon}{h(0)}$ and $w_0 = \frac{w}{\zeta}$ (\emph{cf.}, \eqref{eq:g_choice}), we have the following, analogously to \eqref{eq:g_norm_se}:
{ \allowdisplaybreaks
\begin{align} 
    \|g\|_k^2
        &= c_1^{-1} \int |G(\xi)|^2 \bigg( \frac{2\nu}{l^2} + 4\pi^2 \|\xi\|_2^2 \bigg)^{(\nu + d/2)} d\xi \label{eq:matern_int1} \\
        &= c_1^{-1} a_0^2 w_0^{2d} \int_{\|w_0 \xi\|_2 \le 1} \exp\bigg( - \frac{2}{ 1-\| w_0 \xi\|_2^2 } \bigg)\bigg( \frac{2\nu}{l^2} + 4\pi^2 \|\xi\|^2 \bigg)^{(\nu + d/2)} d\xi \label{eq:matern_int2} \\
        &\le c_1^{-1} a_0^2 w_0^{2d} \int_{\|\xi\|_2 \le \frac{1}{w_0}}\bigg( \frac{2\nu}{l^2} + 4\pi^2 \|\xi\|_2^2 \bigg)^{(\nu + d/2)} d\xi \label{eq:matern_int3} \\
        &\le c_1^{-1} a_0^2 w_0^{2d} V(w_0^{-1}) \bigg( \frac{2\nu}{l^2} + \frac{4\pi^2}{w_0^2} \bigg)^{(\nu + d/2)}  \label{eq:matern_int4} \\
        &= c_1^{-1} a_0^2 w_0^{d-2\nu} V(w_0^{-1}) \bigg( \frac{2\nu w_0^2}{l^2} + 4\pi^2 \bigg)^{(\nu + d/2)} \label{eq:matern_int5}  \\
        &\le c_2 a_0^2 w_0^{-2\nu} \bigg( \frac{2\nu w_0^2}{l^2} + 4\pi^2 \bigg)^{(\nu + d/2)} \label{eq:matern_int6} \\
        &\le c_2 a_0^2 w_0^{-2\nu} \big(8\pi^2\big)^{(\nu + d/2)}, \label{eq:matern_int7} 
\end{align} }
where:
\begin{itemize}
    \item \eqref{eq:matern_int2} follows from \eqref{eq:Hf}--\eqref{eq:g_choice} and the fact that the $d$-dimensional Fourier transform of $a h(x/b)$ is $ab^d H(xb)$;
    \item \eqref{eq:matern_int3} follows by upper bounding the exponential term by one;
    \item \eqref{eq:matern_int4} follows by upper bounding $\|\xi\|_2^2 \le \frac{1}{w_0^2}$ according to the constraint, and using the definition of the volume in \eqref{eq:volume};
    \item \eqref{eq:matern_int6} holds for some $c_2 > 0$ since $V(w_0^{-1}) = O(w_0^{-d})$ by \eqref{eq:volume};
    \item \eqref{eq:matern_int7} holds when $w_0$ is chosen such that $\frac{2\nu w_0^2}{l^2} \le 4\pi^2$, as will be done shortly.
\end{itemize}

Equating \eqref{eq:matern_int7} with $B^2$ (thus ensuring $\|g\|_k \le B$), and recalling that $a_0 = \frac{2\epsilon}{h(0)}$, we obtain
\begin{align}
    \frac{B c_2^{-1/2} h(0) w_0^{\nu}}{ \big( 8\pi^2 \big)^{(\nu + d/2)/2} } = 2\epsilon. \label{eq:matern_pre}
\end{align}
We again observe that in the limit as $\frac{\epsilon}{B} \to 0$ we have $w_0 \to 0$; hence, to satisfy the above condition $\frac{2\nu w_0^2}{l^2} \le 4\pi^2$, it suffices that $\frac{\epsilon}{B}$ is sufficiently small.

Rearranging \eqref{eq:matern_pre}, we obtain
\begin{align}
    w_0 &= \Bigg( 2\epsilon \frac{ (8\pi^2)^{(\nu + d/2)/2} }{ B c_2^{-1/2} h(0) } \Bigg)^{1/\nu}. \label{eq:matern_post}
\end{align}
Recalling that $w = \zeta w_0$ in \eqref{eq:M_choice} and \eqref{eq:g_choice}, we find that \eqref{eq:matern_post} gives the following analog of \eqref{eq:M_se}:
\begin{equation}
    M = \Big\lfloor \Big( \frac{B c_3}{\epsilon} \Big)^{d/\nu} \Big\rfloor, \label{eq:M_matern}
\end{equation}
where 
\begin{equation}
    c_3 :=  \Big( \frac{1}{\zeta} \Big)^{\nu} \cdot \Bigg( \frac{ c_2^{-1/2} h(0) }{ 2 (8\pi^2)^{(\nu + d/2)/2} } \Bigg).
\end{equation}

\section{Bounding the Regret}

In this section, we combine the tools from Sections \ref{sec:ensemble} and \ref{sec:BOUND_RKHS} to deduce the regret bounds given in Theorems \ref{thm:instantaneous} and \ref{thm:cumulative}.  Throughout the section, we use the fact that for $P_1$ and $P_2$ representing the density functions of Gaussian random variables with means $(\mu_1,\mu_2)$ and variance $\sigma^2$, we have
\begin{equation}
    D(P_1 \| P_2) = \frac{ (\mu_1 - \mu_2)^2 }{ 2\sigma^2 }. \label{eq:Gaussian_div}
\end{equation}  
Here, $D(P_1 \| P_2) = \int_{\RR} P_1(z) \log \frac{P_1(z)}{P_2(z)} dz$ denotes the Kullback-Leibler divergence between two density functions \cite{Cov01}.

\subsection{Auxiliary Lemmas and Definitions}

Fix an arbitrary (e.g, optimal) bandit optimization algorithm.  Let $\yv = (y_1,\dotsc,y_T)$ be the observations up to time $T$, and for $m=1,\dotsc,M$, let $P_m(\yv)$ be the probability density function of $\yv$ upon running the algorithm on a given function $f_m$ indexed by $m$.  Moreover, let $P_0(\yv)$ be the probability density of $\yv$ when the algorithm is run with $f$ being zero everywhere.  
Let $\EE_m$ and $\PP_m$ denote expectations and probabilities when the underlying function is $f = f_m$, and similarly for $\EE_0$ and $\PP_0$.  Finally, let $\EE[\cdot] = \frac{1}{M}\sum_{m=1}^M \EE_m[\cdot]$ be the expectation averaged over a uniformly random function index, and similarly for $\PP[\cdot]$.

The following lemma from \cite{Aue95} relates two expectations $\EE_m$ and $\EE_0$ in terms of the corresponding divergence $D(P_0\|P_m)$.

\begin{lemma} \emph{\cite[p.~27]{Aue95}} \label{lem:auer}
	For any function $a(\yv)$ taking values in a bounded range $[0,A]$, we have
	\begin{equation}
		\EE_m[a(\yv)] \le \EE_0[a(\yv)] + A \sqrt{ D(P_0 \| P_m) }. \label{eq:auer_bound}
	\end{equation}
\end{lemma}
\noindent To make the paper self-contained, the proof is given in the appendix.

Before proceeding, we introduce some additional notation:
\begin{itemize}
	\item Let $\{\Rc_m\}_{m=1}^M$ be the partition of the domain into $M$ regions according to \eqref{eq:M_choice}, i.e., done according to a uniform grid with $f_m$ taking the maximum in the centre of $\Rc_m$;
	\item Let $j_t$ be the index at time $t$ such that $x_t$ falls into $\Rc_{j_t}$ -- this can be thought of as a quantization of $x_t$;
	\item Let $N_j = \sum_{t=1}^T \openone\{j_t = j\}$ denote the number of points from $\Rc_j$ that are selected throughout the $T$ rounds;
	\item Let $P_m(y_t|\yv_{t-1})$ denote the distribution of $y_t$ given all the observations up to time $t-1$, denoted by $\yv_{t-1} = (y_1,\dotsc,y_{t-1})$ (and $\yv_0 = \emptyset$), in the case that $f = f_m$;
	\item Define the maximum function value within a single region $\Rc_j$ as
	\begin{equation}
		\vbar_m^j := \max_{x \in \Rc_j} f_m(x), \label{eq:vbar}
	\end{equation}
	and the maximum divergence within the region as
	\begin{equation}
		\Dbar_m^j := \max_{x \in \Rc_j} D( P_0(\cdot|x) \| P_m(\cdot|x) ), \label{eq:Dbar}
	\end{equation}
	where $P_m(y|x)$ is the distribution of an observation $y$ for a given selected point $x$ under the function $f_m$, and similarly for $P_0(y|x)$.
\end{itemize}

\noindent The following lemma provides a simple upper bound on the divergence in \eqref{eq:auer_bound}. 

\begin{lemma}
    Under the preceding definitions, we have
    \begin{equation}
        D(P_0 \| P_m) \le \sum_{j=1}^M \EE_0[N_j]\Dbar_m^j \label{eq:div_bound}
    \end{equation}
\end{lemma}
\begin{proof}
    We use the chain rule for divergence \cite[Sec.~2.5]{Cov01} to write 
    {\allowdisplaybreaks \begin{align}
    	D(P_0 \| P_m) 
    		&= \sum_{t=1}^T \EE_{0} \Big[ D\big( P_0(\cdot|\yv_{t-1}) \| P_m(\cdot|\yv_{t-1}) \big) \Big]  \label{eq:div_bound_1} \\
    		&\le \sum_{t=1}^T \sum_{j=1}^M \PP_0[ j_t = j ] \Dbar_m^j \label{eq:div_bound_2} \\
    		&= \sum_{j=1}^M \bigg(\sum_{t=1}^T \PP_0[ j_t = j ] \bigg) \Dbar_m^j \\
    		&= \sum_{j=1}^M \EE_0[N_j]\Dbar_m^j, \label{eq:div_bound_3}
    \end{align} }
    where \eqref{eq:div_bound_2} follows by noting that the argument to the expectation in \eqref{eq:div_bound_1} depends on $\yv_{t-1}$ only through the resulting selected point $x_t$, which is chosen based on $\yv_{t-1}$.  When $x_t$ falls into $\Rc_j$ (i.e., $j_t = j$), the divergence is upper bounded by $\Dbar_m^j$ due to \eqref{eq:Dbar}.
\end{proof}

Finally, the following technical lemma will also be key to the analysis.  It roughly states that if we sum the function values $f_m(x)$ over $x$ lying on the grid defined by $\{\Rc_j\}$, then the total is dominated by the largest value, i.e., it behaves as $O(\epsilon)$.

\begin{lemma} \label{lem:vbar_sums}
	The functions $\{f_m\}$ constructed in Section \ref{sec:FUNC_CLASS} are such that $\vbar_m^j$ satisfy the following:
	\begin{enumerate}
		\item $\sum_{j=1}^M \vbar_m^j = O(\epsilon)$ for all $m$;
		\item $\sum_{m=1}^M \vbar_m^j = O(\epsilon)$ for all $j$;
		\item $\sum_{m=1}^M (\vbar_m^j)^2 = O(\epsilon^2)$ for all $j$.
	\end{enumerate}
\end{lemma}
\begin{proof}
    The three claims in the lemma statement are proved in a nearly identical manner, so we focus on the first.  Recall from \eqref{eq:M_choice} that each $f_m$ is a shifted version of $g(x) = a_0 h\big(\frac{x}{\zeta w}\big)$, where the shifts among the $m$ values differ by integer multiples of $w$ in one or more of the $d$ dimensions.  That is, the amount by which $g$ is shifted to produce $f_m$ is given by $w \iv_m$ for some $\iv_m = (\iv_{m,1},\dotsc,\iv_{m,d}) \in \ZZ^d$, and the vectors $\iv_{m}$ and $\iv_{m'}$ differ in at least one coordinate when $m \ne m'$.  
    
    In the following, we let $\bone$ denote the vector of ones, and we write $\iv \preceq \iv'$ for element-wise inequalities. Recalling the definitions of $\Rc_j$ and $\vbar_m^j$ following Lemma \ref{lem:auer}, we have
    \begin{align}
        \sum_{j=1}^M \vbar_m^j 
            &= \sum_{j=1}^M \max_{x \in \Rc_j} f_m(x) \\
            &\le \sum_{\iv \in \ZZ^m} \max_{\iv w \preceq x \preceq (\iv+\bone)w} g(x) \label{eq:sum_vbar2} \\
            &= \frac{2\epsilon}{h(0)} \sum_{\iv \in \ZZ^m} \max_{\iv \preceq x' \preceq \iv+\bone} h(\zeta x'), \label{eq:sum_vbar3}
    \end{align}
    where \eqref{eq:sum_vbar2} follows by expanding the sum over $\iv \in \{\iv_m \,:\,m = 1,\dotsc,M\}$ to all $\iv \in \ZZ^d$ and recalling that each $f_m$ is a shifted version of $g$, and \eqref{eq:sum_vbar3} follows by substituting \eqref{eq:g_choice} and applying the change of variable $x' = \frac{x}{w}$.
    
    Since $\frac{2\epsilon}{h(0)} = O(\epsilon)$, it only remains to show that the summation in \eqref{eq:sum_vbar3} is $O(1)$.  To do this, we note that since the bump function $H(\xi)$ in \eqref{eq:Hf} is infinitely differentiable, its inverse Fourier transform $h(x)$ decays to zero as $\|x\|_2 \to \infty$, at a rate faster than any finite power of $\frac{1}{\|x\|_2}$ \cite{Mea73}.  The $2d$-th power suffices for our purposes: There exist constants $C_1,C_2$ such that $h(\zeta x') \le \frac{C_1}{\zeta^2 \|x'\|_2^{2d}}$ whenever $\|x'\|_2^2 \ge C_2$.  Hence, we bound the sum in \eqref{eq:sum_vbar3} by considering two separate cases:
    \begin{itemize}
        \item If the vector $\iv$ is such that $\|x'\|_2^2 < C_2$ for some $\iv \preceq x' \preceq (\iv+\bone)$, we upper bound $h(\zeta x')$ by its maximum value $h(0)$.  This case can only occur for a finite number of $\iv$, and hence the total contribution from such $\iv$ is $O(1)$ (recall that we assume $d = O(1)$).
        \item For the vectors $\iv$ is such that $\|x'\|_2^2 \ge C_2$ for all $\iv \preceq x' \preceq (\iv+\bone)$, we apply the upper bound $h(\zeta x') \le \frac{C_1}{\zeta^2 \|x'\|_2^{2d}}$.  The total contribution is again finite, since the sum of $\frac{1}{\|\iv\|_2^{2d}}$ over all $\iv \ne \bzero$ is finite.  For instance, this can be seen by upper bounding $\frac{1}{\|\iv\|_2^{2d}} \le \prod_{l=1}^d \frac{1}{i_l^2}$ (since $\|\iv\|_2^2 \ge i_l^2$), which implies that $\sum_{\iv \succeq \bone} \frac{1}{\|\iv\|_2^{2d}} \le \big(\sum_{i=1}^{\infty} \frac{1}{i^2}\big)^d < \infty$.
    \end{itemize}
    Combining the above, we conclude that $\sum_{j=1}^M \vbar_m^j = O(\epsilon)$.
\end{proof}

\subsection{Completion of the Proof of Theorem \ref{thm:instantaneous}} \label{sec:simple}


{\bf Applying Lemma \ref{lem:auer}:} Letting $v^{(T)} = f(x^{(T)})$ be the reward of the recommended point, we have
\begin{align}
	\EE_{m}[v^{(T)}] 
		&= \EE_m[ f(x^{(T)}) ] \\
		&\le \sum_{j=1}^M \PP_m[x^{(T)} \in \Rc_j] \vbar_m^j \label{eq:VT_bound_3} \\
		&\le \sum_{j=1}^M \vbar_m^j \Bigg( \PP_0[ x^{(T)} \in \Rc_j ] + \sqrt{ \sum_{j'=1}^M \EE_0[N_{j'}]\Dbar_m^{j'} } \Bigg) \label{eq:VT_bound_5}
\end{align}
where \eqref{eq:VT_bound_3} follows by the same argument as \eqref{eq:div_bound_2}, and \eqref{eq:VT_bound_5} follows from Lemma \ref{lem:auer} and \eqref{eq:div_bound}, with the former setting $a(\yv) = \openone\{ x^{(T)} \in \Rc_j \}$ and using the fact that $\openone\{\cdot\} \in [0,1]$ (note that $x^{(T)}$ is a function of $\yv$).  Averaging over $m$ gives the following bound on the average reward:
\begin{equation}
	\EE[v^{(T)}] \le \frac{1}{M}\sum_{m=1}^M \sum_{j=1}^M \vbar_m^j \Bigg( \PP_0[ x^{(T)} \in \Rc_j ] + \sqrt{ \sum_{j'=1}^M \EE_0[N_{j'}]\Dbar_m^{j'} } \Bigg). \label{eq:avg_VT}
\end{equation}

\textbf{Bounding the two terms in \eqref{eq:avg_VT}:} Using the second part of Lemma \ref{lem:vbar_sums}, we can bound the first term in \eqref{eq:avg_VT} as follows:
\begin{align}
	\frac{1}{M}\sum_{m=1}^M \sum_{j=1}^M \vbar_m^j \PP_0[ x^{(T)} \in \Rc_j ] 
    	&= \frac{1}{M}\sum_{j=1}^M \bigg(\sum_{m=1}^M  \vbar_m^j\bigg) \PP_0[ x^{(T)} \in \Rc_j ]  \\
    	& = O\bigg(\frac{\epsilon}{M}\bigg) \sum_{j=1}^M \PP_0[ x^{(T)} \in \Rc_j ] \\
    	& = O\bigg(\frac{\epsilon}{M}\bigg), \label{eq:first_term_3}
\end{align}
since $\sum_{j=1}^M \PP_0[ x^{(T)} \in \Rc_j ] = 1$.

The second term in \eqref{eq:avg_VT} is bounded as follows: {\allowdisplaybreaks
\begin{align}
	& \frac{1}{M}\sum_{m=1}^M \sum_{j=1}^M \vbar_m^j \sqrt{ \sum_{j'=1}^M \EE_0[N_{j'}]\Dbar_m^{j'} } \nonumber \\
	& \qquad =  \frac{1}{\sqrt{2} \cdot \sigma}\cdot\frac{1}{M}\sum_{m=1}^M \bigg(\sum_{j=1}^M \vbar_m^j\bigg) \sqrt{ \sum_{j'=1}^M \EE_0[N_{j'}](\vbar_m^{j'})^2 } \label{eq:second_term_2} \\
	& \qquad= O\bigg(\frac{\epsilon}{\sigma}\bigg)\cdot\frac{1}{M}\sum_{m=1}^M \sqrt{ \sum_{j'=1}^M \EE_0[N_{j'}](\vbar_m^{j'})^2 } \label{eq:second_term_3} \\
	&\qquad\le O\bigg(\frac{\epsilon}{\sigma}\bigg)\cdot\sqrt{ \frac{1}{M}\sum_{m=1}^M \sum_{j'=1}^M \EE_0[N_{j'}](\vbar_m^{j'})^2 } \label{eq:second_term_4} \\
	&\qquad =O\bigg(\frac{\epsilon}{\sigma}\bigg)\cdot\sqrt{ \frac{1}{M} \sum_{j'=1}^M \EE_0[N_{j'}]\bigg( \sum_{m=1}^M (\vbar_m^{j'})^2\bigg) } \label{eq:second_term_5} \\
	&\qquad =O\bigg(\frac{\epsilon^2}{\sqrt{M} \sigma}\bigg)\cdot\sqrt{ \sum_{j'=1}^M \EE_0[N_{j'}] } \label{eq:second_term_6} \\
	&\qquad = O\bigg(\frac{\sqrt{T}\epsilon^2}{\sqrt{M} \sigma}\bigg), \label{eq:second_term_7}
\end{align} }
where \eqref{eq:second_term_2} follows since the divergence associated with a point $x$ having value $v(x)$ is $\frac{v(x)^2}{2\sigma^2}$ (\emph{cf.}, \eqref{eq:Gaussian_div}), \eqref{eq:second_term_3} follows from the first part of Lemma \ref{lem:vbar_sums}, \eqref{eq:second_term_4} follows from Jensen's inequality, \eqref{eq:second_term_6} follows from the third part of Lemma \ref{lem:vbar_sums}, and \eqref{eq:second_term_7} follows from $\sum_{j'} N_{j'} = T$.

{\bf Combining and Simplifying}: Substituting \eqref{eq:first_term_3} and \eqref{eq:second_term_7} into \eqref{eq:avg_VT} gives
\begin{equation}
	\EE[v^{(T)}] \le C\cdot \epsilon \bigg(\frac{1}{M} + \frac{\epsilon}{
	\sigma}\sqrt{\frac{T}{M}}\bigg)
\end{equation}
for some constant $C$.  Since the maximum function value is $f(x^*) = 2\epsilon$ by the construction in Section \ref{sec:FUNC_CLASS}, this implies that the simple regret is lower bounded by
\begin{align}
\EE[r^{(T)}] 
	&\ge 2\epsilon -  C\cdot \epsilon \bigg(\frac{1}{M} + \frac{\epsilon}{\sigma}\sqrt{\frac{T}{M}}\bigg) \\
	&= \epsilon\bigg( 2 - \frac{C}{M} - \frac{C\epsilon}{\sigma}\sqrt{\frac{T}{M}} \bigg). 
\end{align}
Noting that $M \to \infty$ as $\frac{\epsilon}{B} \to 0$ in both \eqref{eq:M_se} and \eqref{eq:M_matern}, we have for sufficiently small $\frac{\epsilon}{B}$ that $\frac{C}{M} \le \frac{1}{2}$, and hence 
\begin{equation}
    \EE[r^{(T)}] \ge \epsilon\bigg( \frac{3}{2} - \frac{C\epsilon}{\sigma}\sqrt{\frac{T}{M}} \bigg). \label{eq:RT_LB}
\end{equation}
By equating the bracketed term with one, it follows that if the time horizon satisfies
\begin{equation}
	T \le \frac{M \sigma^2}{4C^2 \epsilon^2}, \label{eq:T_final}
\end{equation}
then the average simple regret at time $T$ is at least $\epsilon$.  

Theorem \ref{thm:instantaneous} now follows by substituting $M = \Theta\big( \big(\log\frac{B}{\epsilon}\big)^{d/2} \big)$ into \eqref{eq:T_final} for the squared exponential kernel (\emph{cf.}, \eqref{eq:M_se}), and $M = \Theta\big(\big( \frac{B}{\epsilon} \big)^{d/\nu}\big)$for the Mat\'ern kernel (\emph{cf.}, \eqref{eq:M_matern}).

\subsection{Completion of the Proof of Theorem \ref{thm:cumulative}}

We initially follow similar steps to those used above for the simple regret, but the later steps become slightly more involved. 

{\bf Applying Lemma \ref{lem:auer}:} Letting $V_T = \sum_{t=1}^T f(x_t)$ be the total cumulative reward, we have
\begin{align}
	\EE_{m}[V_T] 
		&= \sum_{t=1}^T \EE_m[ f(x_t) ] \\
		&\le \sum_{t=1}^T \sum_{j=1}^M \PP_m[j_t = j] \vbar_m^j \label{eq:VT_bound_3c} \\
		&= \sum_{j=1}^M \vbar_m^j \EE_m[N_j] \\
		&\le \sum_{j=1}^M \vbar_m^j \Bigg( \EE_0[N_j] + T \sqrt{ \sum_{j'=1}^M \EE_0[N_{j'}]\Dbar_m^{j'} } \Bigg), \label{eq:VT_bound_5c}
\end{align}
where \eqref{eq:VT_bound_3c} follows by the same argument as \eqref{eq:div_bound_2}, and \eqref{eq:VT_bound_5c} follows from Lemma \ref{lem:auer} and \eqref{eq:div_bound}, with the former setting $a(\yv) = N_j$ and using the fact that $N_j \in [0,T]$ (note that $\{x_t\}_{t=1}^T$ is a function of $\yv$).  Averaging over $m$ gives the following bound on the average total reward:
\begin{equation}
	\EE[V_T] \le \frac{1}{M}\sum_{m=1}^M \sum_{j=1}^M \vbar_m^j \Bigg( \EE_0[N_j] + T \sqrt{ \sum_{j'=1}^M \EE_0[N_{j'}]\Dbar_m^{j'} } \Bigg). \label{eq:avg_VTc}
\end{equation}

\textbf{Bounding the two terms in \eqref{eq:avg_VTc}:}  Using the second part of Lemma \ref{lem:vbar_sums}, we can bound the first term in \eqref{eq:avg_VTc} as follows:
\begin{align}
	\frac{1}{M}\sum_{m=1}^M \sum_{j=1}^M \vbar_m^j \EE_0[N_j] &= \frac{1}{M}\sum_{j=1}^M \bigg(\sum_{m=1}^M  \vbar_m^j\bigg) \EE_0[N_j] \\
	& = O\bigg(\frac{\epsilon}{M}\bigg) \sum_{j=1}^M \EE_0[N_j] \\
	& = O\bigg(\frac{T\epsilon}{M}\bigg), \label{eq:first_term_3c}
\end{align}
since $\sum_{j=1}^M N_j = T$.  Moreover, the second term in \eqref{eq:avg_VTc} was already bounded in \eqref{eq:second_term_2}--\eqref{eq:second_term_7}:
\begin{equation}
    T\cdot\frac{1}{M}\sum_{m=1}^M \sum_{j=1}^M \vbar_m^j \sqrt{ \sum_{j'=1}^M \EE_0[N_{j'}]\Dbar_m^{j'} }= O\bigg(\frac{T\sqrt{T}\epsilon^2}{\sqrt{M} \sigma}\bigg), \label{eq:second_term_7c}
\end{equation}

\textbf{Combining and simplifying:} Substituting \eqref{eq:first_term_3c} and \eqref{eq:second_term_7c} into \eqref{eq:avg_VTc} gives
\begin{equation}
	\EE[V_T] \le C'\cdot T\epsilon \bigg(\frac{1}{M} + \frac{\epsilon}{
	\sigma}\sqrt{\frac{T}{M}}\bigg)
\end{equation}
for some constant $C'$.  Since the maximum function value is $f(x^*) = 2\epsilon$ by the construction in Section \ref{sec:FUNC_CLASS}, this implies that the cumulative regret is lower bounded by
\begin{align}
\EE[R_T] 
	&\ge 2T\epsilon -  C'\cdot T\epsilon \bigg(\frac{1}{M} + \frac{\epsilon}{\sigma}\sqrt{\frac{T}{M}}\bigg) \\
	&= T\epsilon\bigg( 2 - \frac{C'}{M} - \frac{C'\epsilon}{\sigma}\sqrt{\frac{T}{M}} \bigg). 
\end{align}
Noting that $M \to \infty$ as $\frac{\epsilon}{B} \to 0$ in both \eqref{eq:M_se} and \eqref{eq:M_matern}, we have for sufficiently small $\frac{\epsilon}{B}$ that $\frac{C'}{M} \le \frac{1}{2}$, and hence 
\begin{equation}
    \EE[R_T] \ge T\epsilon\bigg( \frac{3}{2} - \frac{C'\epsilon}{\sigma}\sqrt{\frac{T}{M}} \bigg). \label{eq:RT_LBc}
\end{equation}
By equating the bracketed term with one, it follows that if the time horizon satisfies
\begin{equation}
	T \le \frac{M \sigma^2}{4(C')^2 \epsilon^2}, \label{eq:T_final_c}
\end{equation}
then the average cumulative regret at time $T$ is at least $T\epsilon$.  For convenience, we also note the following equivalent form of \eqref{eq:T_final_c}:
\begin{equation}
	\epsilon \le \sqrt{\frac{M \sigma^2}{4(C')^2 T}}, \label{eq:T_final2c}
\end{equation}

While $\epsilon$ represented the target simple regret in Section \ref{sec:simple}, for cumulative regret it can be treated as a parameter that we can set to our liking, subject to the fact that we assumed in the analysis that $\frac{\epsilon}{B}$ is sufficiently small.  Hence, we simply need to choose $\epsilon$ to ensure that \eqref{eq:T_final_c} (or equivalently, \eqref{eq:T_final2c}) holds, and substitute this value into the lower bound $T\epsilon$.

To do this, we choose $\epsilon$ such that \eqref{eq:T_final2c} nearly holds with equality:
\begin{equation}
    \epsilon \ge \sqrt{\frac{M\sigma^2}{8 (C')^2 T}}. \label{eq:eps_choice}
\end{equation}
Note that the right-hand side depends on $M$, while $M$ itself is expressed in terms of $\epsilon$ in \eqref{eq:M_se} and \eqref{eq:M_matern}. Hence, it remains to show that these choices are consistent, and in particular, to lower bound $\epsilon$ (and hence the cumulative regret bound $T\epsilon$) in terms of $T$ alone.

{\bf Application to the SE kernel:} For the SE kernel, we have from the choice $M = \Theta\big( \big(\log\frac{B}{\epsilon}\big)^{d/2} \big)$ in \eqref{eq:M_se}, along with the upper and lower bounds on $\epsilon$ in \eqref{eq:T_final2c} and \eqref{eq:eps_choice}, that 
\begin{equation}
    \epsilon = \Theta\bigg( \sqrt{\frac{\sigma^2}{T} \Big( \log\frac{B}{\epsilon} \Big)^{d/2} } \bigg). \label{eq:eps_choice2}
\end{equation}
We therefore deduce that
\begin{equation}
    \log\frac{B}{\epsilon} = \log\sqrt{\frac{TB^2}{\sigma^2}} - \log\bigg( \Theta(1) \Big( \log\frac{B}{\epsilon} \Big)^{d/4} \bigg) \label{eq:log_Be2}
\end{equation}
by a direct substitution of \eqref{eq:eps_choice2}.  Now, since we have assumed that $d = O(1)$, we find that the second term behaves as $O\big(\log\log\frac{B}{\epsilon}\big)$, which is at most $\frac{1}{2}\log\frac{B}{\epsilon}$ when $\frac{\epsilon}{B}$ is sufficiently small.  By moving this term to the left-hand side of \eqref{eq:log_Be2}, we find that $\log\frac{B}{\epsilon} = \Theta\big( \log\frac{B^2 T}{\sigma^2} \big)$, and substitution into \eqref{eq:eps_choice2} yields  $\epsilon = \Theta\big( \sqrt{\frac{\sigma^2}{T} \big( \log\frac{B^2 T}{\sigma^2} \big)^{d/2} } \big)$, meaning that the lower bound $T\epsilon$ yields \eqref{eq:cumul_SE}.  

Note that the behavior $\epsilon = \Theta\big( \sqrt{\frac{\sigma^2}{T} \big( \log\frac{B^2 T}{\sigma^2} \big)^{d/2} } \big)$ reveals that $\frac{\epsilon}{B}$ is indeed arbitrary small when the implied constant in the theorem assumption $\frac{\sigma}{B} = O(\sqrt{T})$ is sufficiently small (note that $z\big(\log\frac{1}{z}\big)^d \to 0$ as $z \to 0$).  This justifies the assumption of $\frac{\epsilon}{B}$ being sufficiently small throughout our analysis.

{\bf Application to the Mat\'ern kernel:} For the Mat\'ern kernel, we have from \eqref{eq:M_matern} that $M = \Theta\big( \big(\frac{B}{\epsilon}\big)^{d/\nu} \big)$, and substitution into \eqref{eq:eps_choice} gives $\epsilon^2 = \Theta\big( \frac{\sigma^2 B^{d/\nu}}{T\epsilon^{d/\nu} } \big)$, or equivalently $T = \Theta\big( \frac{\sigma^2 B^{d/\nu} }{\epsilon^{2+d/\nu}} \big)$.  This, in turn, implies that $\epsilon = \Theta\big( \sigma^{\frac{2}{2+d/\nu}} B^{\frac{d/\nu}{2+d/\nu}} T^{\frac{-1}{2+d/\nu}} \big)$, and hence the lower bound $T\epsilon$ yields \eqref{eq:cumul_Mat}.  Once again, we find that the assumption $\frac{\sigma}{B} = O\big(\sqrt{T}\big)$ (with a sufficiently small implied constant) in the theorem statement ensures that $\frac{\epsilon}{B}$ is sufficiently small.

\subsection{High-Probability Regret Bounds} \label{sec:HIGH_PROB}

Here we discuss how to adapt the above proofs to provide lower bounds on the regret that hold with a given probability $1-\delta$. 

For the simple regret, recall that we considered functions where all points are at most $4\epsilon$-suboptimal (implying $r^{(T)} \le 4\epsilon$), and showed that $\EE[r^{(T)}] \ge \epsilon$ for all $T < T_{\epsilon}$ and suitably-defined $T_{\epsilon}$.  For such $T$, it follows from the reverse Markov inequality (i.e., Markov's inequality applied to the random variable $4\epsilon - r^{(T)}$) that
\begin{equation}
    \PP[r^{(T)} \ge \eta\epsilon] \ge \frac{\epsilon - \eta\epsilon}{4\epsilon - \eta\epsilon} = \frac{1-\eta}{4-\eta}
\end{equation}
for any $\eta \in (0,1)$.  By choosing $\eta$ sufficiently small, and renaming $\eta\epsilon$ as $\epsilon'$, it follows that in order to achieve some target simple regret $\epsilon'$ with any constant probability above $\frac{3}{4}$, a lower bound of the same form as the average regret bound holds.

An analogous argument applies for the cumulative regret upon recalling that we chose $\epsilon$ such that the average cumulative regret was lower bounded by $T\epsilon$, while also considering functions such that the cumulative regret can never exceed $4T \epsilon$.

\section{Conclusion}

We have given, to our knowledge, the first lower bounds on regret for (non-Bayesian) Gaussian process bandit optimization in the presence of noise, considering both the simple regret and the cumulative regret.  These bounds nearly match existing upper bounds for the squared exponential kernel, with the gaps being more significant for the Mat\'ern kernel.  

An immediate direction for future research is to settle our conjecture in Section \ref{sec:results} regarding the presence of $\gamma_T$ vs.~$\gamma_T^2$ in the existing upper bounds, which would close the gaps present for the squared exponential kernel, and bring the bounds for the Mat\'ern kernel closer together.  Another interesting direction is to provide lower bounds and improved upper bounds for the \emph{Bayesian} setting.  In this setting, the ``needle in haystack'' type functions considered in the present paper are extremely unlikely to be observed, and it is thus reasonable to expect that the optimal scaling laws may be significantly milder than those for the RKHS setting.

\appendix

\section{Proof of Lemma \ref{lem:auer}}

We repeat the short proof from \cite{Aue95} here for completeness:
\begin{align}
    \EE_m[a(\yv)] - \EE_0[a(\yv)]
        &= \int a(\yv) \big( P_m(\yv) - P_0(\yv) \big) d\yv  \\
        &\le \int_{P_m(\yv) \ge P_0(\yv)} a(\yv) \big( P_m(\yv) - P_0(\yv) \big) d\yv \label{eq:auer_pf2} \\
        &\le A \int_{P_m(\yv) \ge P_0(\yv)} \big( P_m(\yv) - P_0(\yv) \big) d\yv \label{eq:auer_pf3}  \\
        &= \frac{A}{2} \| P_0 - P_m \|_1 \label{eq:auer_pf4}  \\
        &\le  A \sqrt{ D(P_0 \| P_m) }, \label{eq:auer_pf5} 
\end{align}
where \eqref{eq:auer_pf2}--\eqref{eq:auer_pf3} follow from the assumption that $a(\yv) \in [0,A]$, \eqref{eq:auer_pf4} is a standard property of the $\ell_1$-norm, and \eqref{eq:auer_pf5} follows from Pinsker's inequality.

\section*{Acknowledgment}

This work was supported in part by the European Commission under Grant ERC Future Proof, SNF 200021-146750 and SNF CRSII2-147633, and the `EPFL Fellows' program (Horizon2020 665667).


\end{document}